\documentclass[twocolumn]{article}
\usepackage[width=17cm,height=22cm]{geometry}
\usepackage[english]{babel}
\usepackage[utf8]{inputenc}
\usepackage{fancyvrb}
\usepackage{authblk}
\usepackage{hyperref}
\usepackage{amsmath,amsthm,amssymb} 
\usepackage{bbm}
\usepackage{natbib}
\usepackage{hyperref}
\usepackage{multirow}
\usepackage{graphicx}
\usepackage{float}
\usepackage{color,soul}
\usepackage{pgf,tikz}
\usepackage{pgfplots}
\pgfplotsset{compat=newest}

\DeclareMathAlphabet{\mathpzc}{OT1}{pzc}{m}{it}

\newcommand{\expect}[1]{\mathbb{E} \left[ #1 \right]}

\newcommand{\Cov}{\mathbb{C}\mathrm{ov}}
\newcommand{\Var}{\mathbb{V}\mathrm{ar}}
\renewcommand{\Pr}{\mathbb{P}}

\newtheorem{proposition}{Proposition}[section]
\newtheorem{remark}{Remark}[section]

\makeatletter
\DeclareFontEncoding{LS1}{}{}
\DeclareFontSubstitution{LS1}{stix}{m}{n}
\DeclareMathAlphabet{\mathscr}{LS1}{stixscr}{m}{n}
\makeatother

\title{Variance Reduction in Actor Critic Methods (ACM)}
\author[1,2]{Eric Benhamou\thanks{\{eric.benhamou@\}\{dauphine.fr\}/\{aisquareconnect.com\}}}
\affil[1]{A.I Square Connect, Lamsade PSL}
\begin{document}
\maketitle

\begin{abstract}
After presenting Actor Critic Methods (ACM), we show ACM are control variate estimators. Using the projection theorem, we prove that the Q and Advantage Actor Critic (A2C) methods are optimal in the sense of the $L^2$ norm for the control variate estimators spanned by functions conditioned by the current state and action. This straightforward application of Pythagoras theorem provides a theoretical justification of the strong performance of QAC and AAC most often referred to as A2C methods in deep policy gradient methods. This enables us to derive a new formulation for Advantage Actor Critic methods that has lower variance and improves the traditional A2C method. 
\end{abstract}

\medskip

\noindent\textbf{keywords}: Actor critic method, Variance reduction, Projection, Deep RL.

\section{Introduction}
Recently ACM have emerged as the state of the art methods in Deep Reinforcement Learning (DRL) problems, \cite{Jaderberg_2016} or \cite{Espeholt_2018}. The introduction of Deep Reinforcement Learning methods have enabled to enlarge the scope of RL to a wide variety of domains through trial and error learning: atari games \cite{Mnih_2016} Go : \cite{Silver_2016}, image recognition \cite{Zoph_2017}, physics tasks, including classic problems such as cartpole swing-up, dexterous manipulation, legged locomotion and car driving \cite{Lillicrap_2015}, traffic signal control \cite{Mannion_2016a}, electricity generator scheduling \cite{Mannion_2016b}, water resource management \cite{Mason_2016}, controlling robots in real
environments \cite{Levine_2016} and other transport problem \cite{Talpaert_2019} and control and manipulation tasks for robotics \cite{Barth_2018,Horgan_2018}. \\
Advantage Actor Critic (A2C)  and Asynchronous Advantage Actor Critic (A3C) methods \cite{Mnih_2016} have been originally derived from the REINFORCE algorithm \cite{Williams_1992}. \cite{Sutton_1999} has derived in full generality the policy gradient descent method while \cite{Mnih_2016} (respectively \cite{Babaeizadeh_2017}) have introduced the deep learning approach as well as capacity to distribute computations across CPUs (respectively GPUs). Currently, Deep Advantage Actor Critic methods achieve state of the art performance in Deep RL. The efficiency of these techniques has been mostly verified from an experimental point of view with multiple experiments on test-beds environment such Open AI gym and Atari games. However, theoretical considerations have been lacking sofar to express why these methods outperform other DRL methods. In this paper, we try precisely to answer this challenging question and to provide a theoretical justification for the efficiency of Advantage Actor Critic methods. After presenting the key motivation for A2C methods and in particular the concept of baseline, we prove that there exist optimal baselines according to the $L^2$ norm. We explain the concept of variance reduction and relate Q and Advantage Actor Critic methods to conditional expectations that can be interpreted as projection in the $L^2$ norm of the initial logathmic gradient in REINFORCE.

\section{Related Work}
Presenting and comparing ACM has been the subject of multiple papers such as \cite{Grondman_2012}, but also recent papers like \cite{Jaderberg_2016}  \cite{Lillicrap_2015,Zoph_2017,Barth_2018,Horgan_2018} or \cite{Espeholt_2018}. Actor critic are examined from a baseline point of view and the stream of quoted papers above aims at proving mostly the convergence of these methods seen as a enhancement of REINFORCE.

Another variety of research has been to find efficient optimization from a GPU and CPU perspective as in \cite{Mnih_2016}, \cite{Babaeizadeh_2017} and lately \cite{Espeholt_2018}. \cite{Mnih_2016} have implied an asynchronous version of the A2C algorithm, having a set of actors that learns from a set of critics. Computation workload is spread against multiple CPUs allowing efficient computation parallelism.  \cite{Babaeizadeh_2017} have made further progress by spreading computation across multiple GPUs leveraging the fact that the learning process in actor critic method can be split in a few highly parallel tasks that can be efficiently done with GPUs.  \cite{Espeholt_2018} have designed a new algorithm entitled V trace for learners to be able to learn asynchronously from different workers.

However, the core of the idea of Actor Critic method, namely a variance reduction technique has not been very well examined and quite overlooked. This paper aims at precisely showing the real importance of variance reduction in these methods to emphasize the logic behind these methods and provides new techniques. This extends the work of \cite{Greensmith_2002} that only consider one dimensional control variates. We show that we can do multi dimensional control variates. Our work also extends the work of \cite{Schulmanetal_2016} that provides a generalized advantage estimation using trust region optimization procedure for both the policy and the value function, which are represented by neural networks.

\section{Background}
We consider the standard reinforcement learning framework. 
A learning agent interacts with an envi-ronment $\mathcal{E}$ to decide rational or optimal actions and receives in returns rewards. 
These rewards are not necessarily only positive. These rewards act as a feedback for finding the best action. 
Using the established formalism of Markov decision process, we assume that there exists a discrete time stochastic control process represented by a 4-tuple defined by  $(\mathcal{S}$, $\mathcal{A}$, $P_a, R_a)$ where $\mathcal{S}$ 
is the set of states, $\mathcal{A}$ the set of actions, 
$P_a(s, s') = \Pr(s_{t+1}=s' \mid s_t = s, a_t=a)$ the transition probability that action $a$ in state $s$ 
at time $t$ will lead to state $s'$ and finally, $R_a(s,a)$ the immediate reward received after state $s$ and action $a$. \\

The requirement of a Markovian state is not a strong constraint as we can stack observations to enforce that the Markov property is satisfyed.

Following \cite{Mnih_2016} or \cite{Jaderberg_2016}, we introduce the concept of observations and pile them to coin states. In this setting, the agent perceives at time $t$ an observation $o_t$ along with a reward $r_t$. 
The agent decides an action $a_t$. 
The agent's state $s_t$ is a function of its experience until time $t$, $s_t = f(o_1, r_1, a_1, ..., o_t, r_t)$.
This setup guarantees that states are Markovian. 
In practice, we do not keep track of the full history but only a limited set of historical experiences. The $n$-step return $R_{t:t+n}$ at time $t$ is defined as the discounted sum of rewards, $R_{t:t+n} = \sum_{i=1}^{n} \gamma^i r_{t+i}$, where $\gamma \in [0,1)$ is the discounted factor. \\

The value function is the expected return from state $s$, $V^\pi(s) = \expect{R_{t:\infty} | s_t = s, \pi}$, when actions are selected accorded to a policy $\pi(a|s)$. The action-value function $Q^\pi(s,a) = \expect{R_{t:\infty} | s_t = s, a_t = a, \pi}$ is the expected return following action $a$ from state $s$. The goal of the agent is to find a policy $\pi$ that maximizes the value function.

\section{Variance Reduction}
When states and or actions are in high dimensions, we parametrize our policy by parameters denoted $\theta$. Typically these parameters are the ones of the parameters of a deep network (the weight of all layers of our deep network). The intuition of ACM is to leverage a policy gradient descent with a reduced variance. Let us show this precisely. Recall  that the policy gradient is given by the policy logarithmic  gradient weighted by the sum of future discounted rewards (see \cite{Williams_1992} and \cite{Sutton_1999})
\begin{eqnarray}
\nabla_\theta J(\theta)\!\!\!\!  &= &\!\! \!\! \mathbb{E}_{\tau} \Bigg[  \sum_{t=0}^{T-1} \! \nabla_\theta \log{\pi_\theta}(a_t \!  \!\mid s_t) \underbrace{\!\!\left( \sum_{t^{'}=t+1}^{T} \gamma^{t^{'}-t} r_{t^{'}} \right)\!\!}_{R_t}  \Bigg] \hspace{0.2cm} \nonumber \\
&= & \!\! \!\! \mathbb{E}_{\tau} \! \left[  \sum_{t=0}^{T-1} \nabla_\theta \log{\pi_\theta}(a_t \mid s_t) R_t  \right] \label{policy_gradient}
\end{eqnarray}
where $\tau$ represents a trajectory, $R_t$ the sum of future discounted rewards, and $\gamma \in [0,1)$ the discount factor. Hence, equation \eqref{policy_gradient} shows that we update the policy deep network parameters through Monte Carlo updates computed as an expectation. We should stop for a while on this expectation as this is a critical part in the update. If the estimation of this expectation is not very accurate because of a large variance of our estimator provided by the standard empirical mean, we would incur high variability in our gradient update, hence a slow converging gradient policy method. It therefore makes a lot of sense to see if we could find another expression of our policy gradient with lower variance. This approach of finding a modified expression inside the expectation that has the same expected value but a lower variance is referred to as variance reduction \cite{Hammersley_1964}. A typical method variance reduction method is to use control variate(s) (see for instance \cite{ross_2002}). A lower variance in the gradient will produce less noisy gradient and cause less unstable learning leading to a policy distribution skewing to the optimal direction more rapidly. 

\subsection{Control Variate}
To present control variate, let us make thing quite general. 
Suppose we try to estimate a quantity $\mu$ defined as an expectation $\mu=\mathbb{E}[\hat{m}]$ of an estimator $\hat{m}$
 and suppose we know another statistic (or estimator) $\hat{t}$ 
such that we not only know its expectation $\tau=\mathbb{E}[\hat{t}]$, but also its correlation 
with our initial estimator denoted by $ \rho_{\hat{m},\hat{t}}$. We can build an unbiased and better estimator of $\mu$ as follows. We compute the 'control variate' estimator of $\hat{m}$ by subtracting a zero expectation term: 
$\hat{m^{'}} = \hat{m} - \alpha ( \hat{t} - \tau)$ for $\alpha \in \mathbb{R}$. We have the following control variate proposition

\begin{proposition} \label{prop1} Optimal Control Variate - 
The control variate estimator $\hat{m^{'}}$ is unbiased for any value of $\alpha$. Among all possible values of $\alpha$, the optimal one (in the sense that it produces the estimator with minimum variance) is given by
$$
\alpha^{*} = \frac{\Cov(\hat{m}, \hat{t}\,)}{\Var(\hat{t})} = \frac{\sigma_{\hat{m}}}{\sigma_{\hat{t}}}  \rho_{\hat{m},\hat{t}}
$$
The corresponding control variate estimator is given by $\hat{m^{*}} = \hat{m} - \alpha^{*} ( \hat{t} - \tau)$ and has a variance given by
$$
\Var(\hat{m^{*}})=(1- \rho_{\hat{m},\hat{t}}^2) \sigma_{\hat{m}}^2 \leq \sigma_{\hat{m}}^2
$$
Hence the best control variate estimators are obtained for highly positively correlated  ($\rho_{\hat{m},\hat{t}} \approx 1$) or negatively  correlated ($\rho_{\hat{m},\hat{t}} \approx -1$) control variates
\end{proposition}
\begin{proof}
Refer to Appendix section \ref{proof1}
\end{proof}

Intuitively, the more correlated (positively or negatively) the estimators $\hat{m}$ and $\hat{t}$, the better we can exploit the knowledge of our control variate zero expectation estimator $ \hat{t} - \tau$ to reduce the variance of our initial estimator $\hat{m}$. If we stop for a minute, this is quite trivial. For our control variates, we know the true expectation. If by any chance our estimator is very correlated (positively or negatively) to our control variates, we can exploit this knowledge to correct our estimator. 

As a matter of fact, through control variate, we exploit information about the errors in estimates of known quantities ($ \hat{t} - \tau$) to reduce the error of an estimate of an unknown quantity ($\hat{m}$). We can also analyze the optimal control variate weight $\alpha^{*} = \frac{\Cov(\hat{m}, \hat{t}\,)}{\Var(\hat{t})}$ as a regression coefficient of our unknown estimator $\hat{m}$ over the space of linear combination of the zero expectation estimator $ \hat{t} - \tau$. Hence control variate consists in just focusing on the orthogonal part of our unknown estimator $\hat{m}$ against the space of linear combination of the zero expectation estimator $ \hat{t} - \tau$. If by any chance our unknown estimator is highly correlated to the control variate, this orthogonal part is close to zero and we achieve a much lower variance estimator.

In practice, in deep RL, we know neither the correlation between the two estimators $\rho_{\hat{m},\hat{t}}$ nor the variances of our two estimators: $ \sigma_{\hat{m}}$ or $\sigma_{\hat{t}}$. Hence, instead of being able to define strictly a control variate estimator, we create a pseudo control variate estimator given by 
\begin{equation}\label{control_variate}
\hat{m} - \hat{t}
\end{equation}
provided $\mathbb{E}[\hat{t}] = 0$. The latter formulation takes a control variate coefficient of $\alpha=1$, which implicitly assumes that $\Cov(\hat{m}, \hat{t}\,) \approx \Var(\hat{t})$. We will use this setting to interpret various Actor Critic (AD) methods as control variate estimators for standard AC method for policy gradients. This is summarized by the proposition below:

\begin{proposition} \label{prop2} Actor Critic Methods - The following estimators of the policy gradient are unbiased and can be analyzed as control variate estimators of REINFORCE policy gradient  estimator given by $\mathbb{E}_{\pi_{\theta}} \! \left[  \sum_{t=0}^{T-1} \nabla_\theta \log{\pi_\theta}(s,a) R_t  \right] $:
\begin{itemize}
\item $\mathbb{E}_{\pi_{\theta}} \! \left[  \displaystyle \sum_{t=0}^{T-1} \nabla_\theta \log{\pi_\theta}(s,a) Q(s,a) \right]$ \qquad (Q-AC)
\item $\mathbb{E}_{\pi_{\theta}} \! \left[  \displaystyle \sum_{t=0}^{T-1} \nabla_\theta \log{\pi_\theta}(s,a) A(s,a) \right]$ \qquad (A-AC)
\item $\mathbb{E}_{\pi_{\theta}} \! \left[ \displaystyle  \sum_{t=0}^{T-1} \nabla_\theta \log{\pi_\theta}(s,a) TD(s) \right]$ \qquad  (TD-AC)
\end{itemize}
where the Advantage function (A)  is defined as 
$$
A(s,a) = \mathbb{E}_{\pi_{\theta}}[ r_{t+1} + \gamma V(s_{t+1}) \mid s_{t}=s, a_{t} = a  ] - V(s_{t})
$$ 
and the Temporal Difference  function (TD) as 
$$
TD(s_{t})= r_{t+1} + \gamma V(s_{t+1}) - V(s_{t})
$$
\end{proposition}
\begin{proof}
Refer to Appendix section \ref{proof2}
\end{proof}

Traditionally, methods for solving Reinforcement Learning (RL) are either categorized as policy method if they aim to find the optimal policy $\pi^{\star}$ or as value methods if they aim to find the optimal 'Q' function. Somehow, Q-AC and AAC methods aim to find the optimal policy thanks to gradient ascent computation but use in their gradient ascent term a 'Q' function, making these methods a mix between policy and value methods. This is summarized by figure \ref{figure1}.

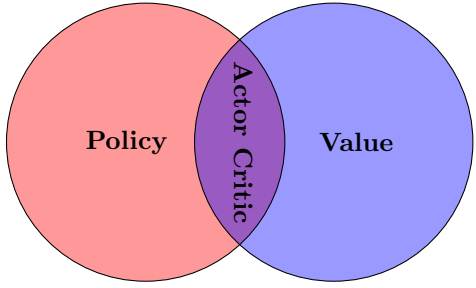
\begin{figure}
\begin{center}
\begin{tikzpicture}
	\begin{scope}[fill opacity=0.4]
      		\fill[red] (0,0) circle (1.85cm);
		\fill[blue] (0:2.5cm) circle (1.85cm);
      		\draw (0,0) circle (1.85cm);
      		\draw (0:2.5cm) circle (1.85cm);
   	\end{scope}
	\draw (-0.25cm,0cm) node { \textbf{Policy}};
	\draw (0cm:2.8cm) node { \textbf{Value}};
	\draw (0cm:1.3cm) node[rotate=-90] { \textbf{Actor Critic}};
\end{tikzpicture}
\end{center}
\caption{The different RL methods. ACM combine policy and value computations, making them mixed methods}\label{figure1}
\end{figure}

\subsection{Conditional expectation, projection and optimality}
Sofar, we have revisited AC methods as control variates. There is however a strong connection with conditional expectation and projection. Let us make the link. But let us first recall a basic property of conditional expectation that states that the conditional expectation with respect to a sub $\sigma$-algebra $\mathcal{G} included in  \mathcal{F}$  of a stochastic variable $X$, $\mathcal{L}^2$ measurable on a probability space $(\Omega, \mathcal{F}, \mathcal{P})$, is the best prediction of the sub-space spanned by this sub $\sigma$-algebra:

\begin{proposition}  \label{prop3}  Conditional expectation and Pythagoras - 
Let $(\Omega, \mathcal{F}, \mathcal{P})$ be a probability space, $X:\Omega \to \mathbb{R}^n$ a random variable on that probability space square integrable and $\mathcal{G} \subseteq \mathcal{F}$ is a sub $\sigma$-algebra of $\mathcal{F}$, then we have that 
\begin{itemize}
\item $X - \mathbb{E}[ X \mid \mathcal{G}]$ is orthogonal to any element $Y$ of $\mathcal{L}^2(\Omega, \mathcal{G}, \mathcal{P})$ where $\mathcal{L}^2(\Omega, \mathcal{G}, \mathcal{P})$ is the space of random variable on the sub  $\sigma$-algebra $(\Omega, \mathcal{G}, \mathcal{P})$ that are square integrable.
\item $\mathbb{E}[ X \mid \mathcal{G}]$  is the best prediction in the sense that $\mathbb{E}[ X \mid \mathcal{G}]$ minimizes its variance with $X$: $\mathbb{E}[ (X - Y)^2]$ among any element $Y \in \mathcal{L}^2(\Omega, \mathcal{G}, \mathcal{P})$.
\end{itemize}
\end{proposition}

\begin{proof}
Refer to Appendix section \ref{proof3}
\end{proof}

Combining proposition \ref{prop1}, \ref{prop2} and \ref{prop3} enables us comparing the various ACM from a variance point of view. It is worth looking at the intuition of the function involved in the different ACM. TD-AC relies on the Temporal Difference which is the projection of the cumulated discounted rewards on the sub $\sigma$-algebra generated by the knowledge of the state $s_{t+1}$ while Q-AC relies on the action value 'Q' function, which is its projection on the sub $\sigma$-algebra generated by the knowledge of the state $s_{t}, a{t}$ and the A-AC on the advantage function which is the difference between the action value 'Q' function and the value function which is of the cumulated discounted rewards on the sub $\sigma$-algebra generated by the knowledge of the state $s_{t}$. Intuitively, as the various sub $\sigma$ algebras are bigger (in the sense that each one is contained by the next one), the corresponding AC methods should become more effective, meaning TD-AC should be improved by Q-AC that should in turn be improved by the A-AC method. This is the subject of proposition \ref{prop4}

\begin{proposition}  \label{prop4}  AC methods Comparison - 
From a variance point of view, TD-AC is less efficient than A-AC, and REINFORCE is less efficient than Q-AC.
\end{proposition}

\begin{proof}
Refer to Appendix section \ref{proof4}
\end{proof}

\section{Towards new AC methods}
In order to create new AC methods, it is useful to notice a property of the policy gradient computation that provides additional control variates.

\begin{proposition}\label{prop5}
If a function $\Phi(s,a )$ is such that when integrating with respect to the policy $\pi_\theta(s,a)$, it does not depend on the parameter $\theta$, which means that 
$\nabla_{\theta} \int  \pi_\theta(s,a) \Phi(s,a)  = 0$, then its gradient policy term is null:
\begin{equation}\label{equality_lemma}
\mathbb{E}_{\pi_{\theta}} \! \left[  \displaystyle \sum_{t=0}^{T-1} \nabla_\theta \log{\pi_\theta}(s,a) \Phi(s,a ) \right] = 0
\end{equation}
\end{proposition}

\begin{proof}
Refer to Appendix section \ref{proof5}
\end{proof}

\begin{remark}\label{remark5}
In particular if the function $\Phi(s,a )$  writes as a function of $s$  only: $\Phi(s,a )=\Psi(s)$ and if $\Psi(s)$ is stationary in the sense that 
$\int  \pi_\theta(s,a) \Psi(s)  d\theta = \Psi(s)$, than its policy gradient term is equal to zero: 
\begin{equation}
\mathbb{E}_{\pi_{\theta}} \! \left[  \displaystyle \sum_{t=0}^{T-1} \nabla_\theta \log{\pi_\theta}(s,a) \Psi(s) \right] = 0
\end{equation}
Typical stationary functions are the state value function $V(s)$ and the constant function: $1$
\end{remark}

Equipped with these additional control variates, it is useful to examine multi dimensional control variates estimators that is the subject of the following proposition

\begin{proposition}  \label{prop6}  Multi Dimensional Control Variates Estimators - 
Let us have a collection of random variables $\hat{t}_1, \ldots, \hat{t}_d$ for which we know the expectation $\tau_i =\mathbb{E}[\hat{t}_i]$ (for any $i=1, \dots, d$). Let us denote by 
$$
T =  (\hat{t}_1 - \mathbb{E}[\hat{t}_1], \ldots, \hat{t}_d - \mathbb{E}[\hat{t}_d])^T
$$
the vector of control variates, $\lambda$ a $d$-dimensional real vector and build the multi dimensional control variates estimator as follows:
\begin{equation}\label{multiDimCV_def}
\hat{m^{'}}  = \hat{m} - \lambda^T T
\end{equation}
As in the one dimensional case, the control variate estimator $\hat{m^{'}}$ is unbiased for any value 
of $\lambda \in \mathbf{R}^d$. Assuming that  $\mathbb{E}[ T\, T^T ]$ is non singular, among all the possible values of $\lambda$, the optimal one (in the sense that it produces the estimator with minimum variance) is given by
\begin{equation}\label{optimal_lambda}
\lambda^{*} = \mathbb{E}[T \,T^T] ^{-1} \mathbb{E}[\hat{m}  T]
\end{equation}
The corresponding control variate estimator is given by $\hat{m^{*}} = \hat{m} - (\lambda^\star)^T T$  and has a variance given by
\begin{equation}\label{optimal_multi_variance}
\Var(\hat{m^{*}})= \Var(\hat{m}) - \mathbb{E}[\hat{m}  T]^T  \mathbb{E}[T \,T^T] ^{-1}  \mathbb{E}[\hat{m}  T]
\end{equation}
\end{proposition}

\begin{proof}
Refer to Appendix section \ref{proof6}
\end{proof}

Equations \eqref{optimal_lambda} and \eqref{optimal_multi_variance} are just generalization of the one dimensional case. The condition that $\mathbb{E}[ T\, T^T ]$ is non singular means that we have control variates that are independent in the sense that none of them is a linear combination of some of the other control variates. This condition is important as to use multi dimensional control variates we really need to find control variates that operate in another dimension space.

\section{Conclusion}
In this paper, we have revisited AC methods. We have shown that these methods can be interpreted as control variate estimators  of REINFORCE. We have also proved using the property of conditional expectation, that the Q and Advantage Actor Critic are optimal control variate estimators. We have invented a new method that combines optimally the general advantage function to establish a new Actor Critic method that is optimal from a control variate point of view. 

%%%%%%%%%%%%%%%%%%%%%%
%%%%%% References %%%%%%%%%%
%%%%%%%%%%%%%%%%%%%%%%
\clearpage
\bibliography{biblio}

\clearpage
\section{Appendix}

%%%%%%%%%%%%%%%%%%%%%%
\subsection{Proof of proposition \ref{prop1}} \label{proof1}
We have 
$$
\mathbb{E}[\hat{m^{'}}] = \mathbb{E}[\hat{m}] - \alpha \mathbb{E}[\hat{t} - \tau]=  \mathbb{E}[\hat{m}] = \mu,
$$
since $\mathbb{E}[\hat{t} ] =\tau$. This proves that the control variate estimator is unbiased for any value of $\alpha$. The variance of this estimator is easy to compute and is given by:
$$
\Var(\hat{m^{'}}) = \Var(\hat{m}) - 2 \alpha \Cov(\hat{m},\hat{t}) + \alpha^2 \Var(\hat{t})  
$$
that is a second order parabola function of $\alpha$ that is minimum for 
$$
\alpha^{*} = \frac{\Cov(\hat{m}, \hat{t}\,)}{\Var(\hat{t})} = \frac{\sigma_{\hat{m}}}{\sigma_{\hat{t}}}  \rho_{\hat{m},\hat{t}}
$$
Its minimum value is given by 
$$
\Var(\hat{m^{*}})=(1- \rho_{\hat{m},\hat{t}}^2) \sigma_{\hat{m}}^2 \leq \sigma_{\hat{m}}^2
$$
that is closed to zero for  highly positively correlated  ($\rho_{\hat{m},\hat{t}} \approx 1$) or negatively  correlated ($\rho_{\hat{m},\hat{t}} \approx -1$) control variates. \qed

%%%%%%%%%%%%%%%%%%%%%%
\subsection{Proof of proposition \ref{prop2}} \label{proof2}
Let us tackle one by one the various estimators given in proposition \ref{prop2}. 

The  \textbf{Q Actor Critic (Q-AC) }estimator is given by
$
\mathbb{E}_{\pi_{\theta}} \! \left[  \displaystyle \sum_{t=0}^{T-1} \nabla_\theta \log{\pi_\theta}(s,a) Q(s,a) \right]
$. It writes also as 
$$
\mathbb{E}_{\pi_{\theta}} \! \left[  \displaystyle \sum_{t=0}^{T-1} \nabla_\theta \log{\pi_\theta}(s,a) (\underbrace{R(s)}_{\hat{m}} - \underbrace{(R(s)-Q(s,a))}_{\hat{t}} \right]
$$
which shows that it is a control variate estimator (as explained in equation \eqref{control_variate}) provided we prove that 
\begin{equation}\label{Q-condition}
\mathbb{E}_{\pi_{\theta}} \! \left[  \displaystyle \sum_{t=0}^{T-1} \nabla_\theta \log{\pi_\theta}(s,a) (Q(s,a)-R(s)) \right] = 0
\end{equation}
The equation \eqref{Q-condition} is trivially verified as one of the definitions of the state action value function $Q(s,a)$ (often referred to as the 'Q' function) is the following:
$$
Q(s,a) = \mathbb{E}_{\pi_{\theta}}[ R(s) \mid s_{t} = s, a_{t}=a]
$$
The law of total expectation states that if $X$ is a random variable and $Y$ any random variable on the same probability space, then
$\mathbb{E}[ \mathbb{E}[ X \mid Y ]] =\mathbb{E}[ X]$ or in other words $\mathbb{E}[ \mathbb{E}[ X \mid Y ] - X ] =0$, which concludes the proof for the Q Actor Critic (Q-AC) method with $X= \displaystyle \sum_{t=0}^{T-1} \nabla_\theta \log{\pi_\theta}(s,a) R(s)$ and $Y=(s_{t} = s, a_{t}=a)$.\\

As for the \textbf{Advantage Actor Critic (A-AC)} method, the same reasoning shows that it suffices to prove that 
$$
\mathbb{E}_{\pi_{\theta}} \! \left[  \displaystyle \sum_{t=0}^{T-1} \nabla_\theta \log{\pi_\theta}(s,a) (A(s,a)-R(s)) \right] = 0
$$
to show that this is also a control variate method. Recall that the advantage function is given by $A(s,a) = Q(s,a) - V(s)$. Using the result previously proved for the Q AC method (equation \eqref{Q-condition}), it suffices to prove that 
\begin{equation}\label{V-condition}
\mathbb{E}_{\pi_{\theta}} \! \left[  \displaystyle \sum_{t=0}^{T-1} \nabla_\theta \log{\pi_\theta}(s,a) V(s) \right] = 0
\end{equation}
But this is a straight consequence of the more general proposition \ref{prop5} and its remark \ref{remark5}.

Finally, let us prove that the \textbf{TD Actor Critic (TD-AC)} method is also a control variate. Recall that the Temporal Difference term is given by 
$$
TD(s_{t})= r_{t+1} + \gamma V(s_{t+1}) - V(s_{t})
$$
Using equation \eqref{V-condition}, it suffices to prove that 
\begin{equation}\label{TD-condition}
\mathbb{E}_{\pi_{\theta}} \! \left[ \displaystyle  \sum_{t=0}^{T-1} \nabla_\theta \log{\pi_\theta}(s,a) ( r_{t+1} + \gamma V(s_{t+1})-R(s_t)) \right] = 0
\end{equation}
This is again a straightforward application of the law of total expectation  as
$$
V(s_{t+1}) =  \mathbb{E}_{\pi_{\theta}} [  \displaystyle  \sum_{s=t+2}^{T} \gamma^{s-(t+2)} r_{s} \mid s_{t+1} ]
$$
while 
$$
R(s_t) =  \displaystyle  \sum_{s=t+1}^{T} \gamma^{s-(t+1)} r_{s}
$$
Hence we can apply the law of total expectation with 
$
X =  \displaystyle  \sum_{t=0}^{T-1} \nabla_\theta \log{\pi_\theta}(s,a)  \displaystyle  \sum_{s=t+2}^{T} \gamma^{s-(t+1)} r_{s} 
$
and $Y=s_{t+1} $. This concludes the proof. \qed

%%%%%%%%%%%%%%%%%%%%%%
\subsection{Proof of proposition \ref{prop3}} \label{proof3}
Let us first prove that for any $Y \in \mathcal{L}^2(\Omega, \mathcal{G}, \mathcal{P})$, we have 
$$
\mathbb{E}[ Y \dot (X-\mathbb{E}[X  \mid \mathcal{G}]] = 0
$$
This is straight application of the law of total expectation (also referred to as the law of iterated expectation or also the tower property) as follows
\begin{eqnarray*}
& & \mathbb{E}[ Y \cdot (X-\mathbb{E}[X  \mid \mathcal{G}])] \\
&=&  \mathbb{E}[   \mathbb{E}[ Y \cdot (X-\mathbb{E}[X  \mid \mathcal{G}] )\mid \mathcal{G}] ] \\
&=&  \mathbb{E}\Bigg[   Y \cdot  \Big( \underbrace{\mathbb{E}[X  \mid \mathcal{G}] - \mathbb{E}[ \mathbb{E}[X  \mid \mathcal{G}]\mid \mathcal{G}] }_{=0}\Big) \Bigg]  \\
& =& 0
\end{eqnarray*}
which proves the orthogonality. \\

For $Y \in \mathcal{L}^2(\Omega, \mathcal{G}, \mathcal{P})$, we have
\begin{eqnarray*}
& & \mathbb{E}[ (X - Y )^2 ] \\
& = & \mathbb{E}[ (X -\mathbb{E}[X  \mid \mathcal{G}] + \mathbb{E}[X  \mid \mathcal{G}] - Y )^2 ] \\
& = & \mathbb{E}[ (X -\mathbb{E}[X  \mid \mathcal{G}]) ^2  + (\mathbb{E}[X  \mid \mathcal{G}] - Y )^2 ]   \\
& & + 2 \underbrace{\mathbb{E}[ (X -\mathbb{E}[X  \mid \mathcal{G}]) (\mathbb{E}[X  \mid \mathcal{G}] - Y )]}_{=0} \\
& \geq  & \mathbb{E}[ (X -\mathbb{E}[X  \mid \mathcal{G}]) ^2  ] 
\end{eqnarray*}
which concludes the proof that the squared $L^2$ norm of $X - Y$ for any element $Y$ of $\mathcal{L}^2(\Omega, \mathcal{G}, \mathcal{P})$ is lower bounded by the squared $L^2$ norm of $X -\mathbb{E}[X  \mid \mathcal{G}]$ \qed

%%%%%%%%%%%%%%%%%%%%%%
\subsection{Proof of proposition \ref{prop4}} \label{proof4}
As explained in the proposition \ref{prop1}, the variance of the control variate is due to the residuals computed as the initial estimator minus the control variate. The control variate is the orthogonal projection of the initial estimator on the linear subspace spanned by the control variate. The proposition \ref{prop3} shows that within the possible sub $\sigma$ the conditional expectation is the orthogonal projection and is the best estimator. Recall that the Temporal difference writes as 
\begin{equation}\label{proof4:TD}
r_{t+1} + \gamma V(s_{t+1}) - V(s_{t})
\end{equation}
Recall also that the value function is a conditional expectation
\begin{equation}\label{proof4:Value}
V(s_{t+1} ) = \mathbb{E}[R_{t+1} \mid s_{t+1}]
\end{equation}
while the state action value 'Q' function is also a conditional expectation but with respect to $ s_{t}, a_{t}$:
\begin{equation}\label{proof4:QFunction}
Q(a_{t}, s_{t} ) = \mathbb{E}[R_{t} \mid s_{t}, a_{t}]
\end{equation}

Last but not least, Advantage function writes as
\begin{equation}\label{proof4:Advantage}
A(a_{t}, s_{t} ) = Q(a_{t}, s_{t}) - V(s_{t})
\end{equation}

From equations \eqref{proof4:TD}, \eqref{proof4:QFunction} and \eqref{proof4:Advantage}, we can deduce that A AC has lower variance than TD AC as the corresponding sub $\sigma$-algebra obtained by conditioning with respect to $ s_{t}, a_{t}$, used for A-AC, contains the one obtained by conditioning with respect to  $s_{t+1}$ and used for TD-AC.

As 'Q' function is a conditional expectation of the future discounted rewards, Q-AC should perform better (have a lower variance) than just REINFORCE.

%%%%%%%%%%%%%%%%%%%%%%
\subsection{Proof of proposition \ref{prop5}} \label{proof5}
\begin{proof}: Recall that the logarithmic gradient writes as
$$
 \nabla_\theta \log{\pi_\theta}(s,a)  = \frac{\nabla_\theta \pi_\theta(s,a)}{\pi_\theta(s,a)}
$$

Assuming smooth functions such that we can interchange expectation (integration) and gradient (derivation), we have the following:
\begin{eqnarray*}
& & \mathbb{E}_{\pi_{\theta}} \! \left[  \displaystyle \sum_{t=0}^{T-1} \nabla_\theta \log{\pi_\theta}(s,a) \Phi(s,a) \right] \hspace{1.5cm} \\
&=&  \displaystyle \sum_{t=0}^{T-1}  \mathbb{E}_{\pi_{\theta}}\left[  \frac{\nabla_\theta \pi_\theta(s,a)}{\pi_\theta(s,a)} \Phi(s,a)  \right] \\
& =& \displaystyle \sum_{t=0}^{T-1}   \int  \frac{\nabla_\theta \pi_\theta(s,a)}{\pi_\theta(s,a)} \Phi(s,a)  \pi_\theta(s,a) d\theta \\
& =& \displaystyle \sum_{t=0}^{T-1}   \int  \nabla_\theta \pi_\theta(s,a) \Phi(s,a)  d\theta \\
&= &\displaystyle \sum_{t=0}^{T-1}   \nabla_\theta  \int  \pi_\theta(s,a) \Phi(s,a)  d\theta \\
&=& 0
\end{eqnarray*}
which concludes the proof\end{proof}

%%%%%%%%%%%%%%%%%%%%%%
\subsection{Proof of proposition \ref{prop6}} \label{proof6}
We can provide multiple proofs and interpretation of this result. 

\textbf{proof 1}
Let us use traditional variation calculus. 
It is immediate that for any value  $\lambda \in \mathbf{R}^d$ it is unbiased as the additional term has a null expectation: $\mathbb{E}[ \lambda^T T ] = 0$. We can compute  the variance of the control variate estimator $\widetilde{m}(\lambda)  = \hat{m} - \lambda^T T$, given by
\begin{equation}
\Var(  \widetilde{m}(\lambda) )  = \Var(\hat{m}) - 2 \lambda^T  \mathbb{E}[\hat{m} T] + \lambda^T   \mathbb{E}[T \,T^T]  \lambda
\end{equation}
Assuming the covariance matrix $ \mathbb{E}[T \,T^T]$ is non singular, our minimum variance problem lies in finding the minimum of a defined parabolla, hence its minimum is given by first order optimality (see \cite{Boyd_2004} A.13)
\begin{equation}
\lambda^* =    \mathbb{E}[T \,T^T]^{-1}  \mathbb{E}[\hat{m} T]
\end{equation}
with the minimum given by:
\begin{equation*}
\Var( \widetilde{m}(\lambda^*) ) =  \Var( \hat{m}) - \mathbb{E}[\hat{m} T]^T \mathbb{E}[T \,T^T]^{-1} \mathbb{E}[\hat{m} T]
\end{equation*}\\

\textbf{proof 2}
Another way to demonstrate this result is to look at the $L^2$ space of all square integrable random variables defined on the same probability space 
as $\hat{m}$ equiped with the canonical inner product $<\!X_1,X_2\!> \!= \!\mathbb{E}[ X_1 X_2] $ and the implied Hilbertian norm 
$\| X \| = \mathbb{E}[ X^2]^{1/2}$ for any $X, X_1, X_2 \in L^{2}$. 
Let $\mathcal{G}$ be the linear compact and closed space subspace 
spanned by any linear combination of $C$: $W \in L^{2}$ such that $W= \mu^T T$
 for some $\mu \in\mathbb{R}^d$. 
Using the fact that the expectation of $\widetilde{m}(\lambda)$ does not depend on $\lambda$, the variance minimum in $\mathbb{R}^d$ can be cast as a minimum distance problem of $\hat{m}$ with $\mathcal{G}$. This is because
\begin{eqnarray*}
&  & \hspace{-0.7cm} \arg \min_{\lambda \in \mathbb{R}^d} \Var( \widetilde{m}(\lambda))  \\
&=& \arg  \min_{\lambda \in \mathbb{R}^d} \mathbb{E}[ \widetilde{m}(\lambda)^2] -   \left( \mathbb{E}[ \widetilde{m}(\lambda) ] \right)^2 \hspace{2.5cm} \\
&=& \arg \min_{\lambda \in \mathbb{R}^d} \mathbb{E}[ \widetilde{m}(\lambda)^2] \\
&=& \arg \min_{\lambda \in \mathbb{R}^d} \| \hat{m} - \lambda^T T \| \\
&=& \arg \min_{g \in \mathcal{G}} \| \hat{m} - g \|
\end{eqnarray*}
This is nice property as the minimum distance problem can be solved with the Hilbert space projection theorem that states that the closest point of $\mathcal{G}$ to $\hat{m}$ is given by its orthogonal projection $g^* \in \mathcal{G}$. This orthogonal projection $g^* = (\lambda^*)^T T$ is characterized uniquely by
$$
< \hat{m} -  g^*, g > = 0 \quad \text{for any } \,\, g  \in \mathcal{G} 
$$
which writes as $\mathbb{E}[ \hat{m} g ] -  \mathbb{E}[ g^*g ] = 0 $. Hence for any $\lambda \in \mathbb{R}^d$, we have
$\mathbb{E}[\hat{m}\, T^T]\lambda = (\lambda^*)^T\mathbb{E}[T \,T^T]\lambda$, which implies
$\mathbb{E}[\hat{m}\, T^T] = (\lambda^*)^T\mathbb{E}[T \,T^T]$
or equivalently 
$$
 (\lambda^*)^T   = \mathbb{E}[\hat{m}\, T^T] \mathbb{E}[T \,T^T] ^{-1} 
$$ leading to the following solution:
\begin{eqnarray*}
 \lambda^* = \mathbb{E}[T \,T^T] ^{-1} \mathbb{E}[\hat{m} T] 
\end{eqnarray*}

\textbf{proof 3}
Let us show how we can decouple the problem into $d$ one dimensional control variate problems using diagonalization. We are looking for the minimum variance control variate spanned by our initial control variates basis $(\hat{t}_1 - \mathbb{E}[\hat{t}_1], \ldots, \hat{t}_d - \mathbb{E}[\hat{t}_d])^T$. Obviously, we can use any equivalent basis. In particular, if we use Gram Schmidt orthogonalization, we can ensure that the $d$ components of our control variate basis are orthogonal for the implied $L^2$ inner product. If some of the component are non independent we can retrieve the corresponding control variate vectors to ensure they are all independent. The orthogonality of the $d$ components of our control variate basis ensures that the covariance matrix $\mathbb{E}[T \,T^T]$ is symmetric and non negative definitive. Hence, because of the Takagi's factorization, there exists a diagonal matrix $D$ with non negative diagonal terms and $U$ an unitary matrix $U U^T = I$ such that $\mathbb{E}[T \,T^T] = U^T D U$. Let define $W=U,T$. We have  $\mathbb{E}[W \,W^T] = D$, so that the components of $W$ are orthogonal with variance given by diagonal terms $W_{ii} = D_{ii} > 0$ for $i=1,\ldots,d$. Because of the equivalence of basis, our optimization problem decouples as it can be reformulated as follows:
\begin{eqnarray*}
\text{minimize}_{\gamma \in \mathbb{R}^d} \Var( \hat{m} - \gamma^T W) 
\end{eqnarray*}
which is equivalent to $\text{minimize}_{\gamma_i \in \mathbb{R}} \Var( \hat{m}_{ii} - \gamma W_{ii})$ for any $i=1, \ldots,d$, where $\hat{m}_{ii}$ is the $i$ coordinate of the random variable $\hat{m} $ in the basis implied by $W$ that is spanned by the orthorgonal vectors $W_i$ that is fill with zero except for coordinate $i$ equal to $W_{ii}$. These decoupled and independent optimization problems are now just one dimensional control variate optimization problem whose solution are 
$$
\gamma_i^* = \frac{\mathbb{E}(\hat{m}_{ii}, W_{ii})}{\mathbb{E}(W_{ii}^2)} =  \frac{<\hat{m},W_i>}{<W_i,W_i>}
$$ so that the solution is given by
$$
 (\gamma^*)^T W = \sum_{i=1,\ldots,d} \gamma_i^* W_{i} = \sum_{i=1,\ldots,d} \frac{<\hat{m},V_i>}{<W_i,W_i>} V_{i}
$$
Noticing that the term $\frac{<\hat{m},V_i>}{<V_i,V_i>}$ is the $i^{th}$ term of the $  \mathbb{E}[\hat{m}\, T^T] \mathbb{E}[T \,T^T] ^{-1} $, this can be rewritten as 
$$
(\gamma^*)^T = \mathbb{E}[\hat{m}\, T^T] \mathbb{E}[T \,T^T] ^{-1}
$$ leading to the following solution:
\begin{eqnarray*}
 \lambda^* = \mathbb{E}[T \,T^T] ^{-1} \mathbb{E}[\hat{m} T] 
\end{eqnarray*}
which concludes the third proof. \qed
\end{document}